\documentclass{article} 
\usepackage{iclr2027_conference,times}

\usepackage[utf8]{inputenc}
\usepackage[T1]{fontenc}

\usepackage{xcolor}

\definecolor{linkpurple}{RGB}{90,40,90}
\definecolor{pdrblue}{RGB}{37,99,235}
\definecolor{green}{HTML}{10B981}

\usepackage{hyperref}       
\hypersetup{
  colorlinks=true,
  linkcolor=linkpurple,
  citecolor=green!50!black,
  urlcolor=blue!50!black
}
 
\usepackage{url}            
\usepackage{booktabs}       
\usepackage{amsfonts}       
\usepackage{nicefrac}       
\usepackage{microtype}      

\usepackage{amsmath,amssymb}
\usepackage{mathtools}
\usepackage{amsthm}
\usepackage{bm}
\usepackage{enumitem}
\usepackage{ulem}
\usepackage{cleveref}

\Crefname{section}{Sec.}{Secs.}

\crefname{appendix}{App.}{Apps.}
\Crefname{appendix}{App.}{Apps.}

\crefname{thm}{Thm.}{Theorems}
\crefname{algorithm}{Alg.}{Algorithms}
\crefname{equation}{Eq.}{Eqs.}
\crefname{figure}{Fig.}{Figs.}
\crefname{prop}{Prop.}{Props.}
\creflabelformat{equation}{#2\textup{#1}#3} 

\makeatletter
\AddToHook{cmd/appendix/before}{%
  \def\cref@section@alias{appendix}%
  \def\cref@subsection@alias{appendix}%
}
\makeatother

\usepackage{pifont}
\usepackage[most]{tcolorbox}

\usepackage{algorithm}
\usepackage{algpseudocode}
\usepackage{placeins}
\usepackage{float}

\usepackage{multirow}
\usepackage{caption} 
\usepackage{adjustbox}
\usepackage{color, colortbl}
\usepackage{subcaption}
\usepackage{graphicx}
\usepackage{makecell}
\usepackage{tikz}
\usetikzlibrary{arrows.meta,positioning}

\usepackage{wrapfig,lipsum}

\newtheorem{proposition}{Proposition}

\newtcolorbox{takeaway}{
  enhanced,
  colback=pdrblue!4,
  colframe=pdrblue!70!black,
  boxrule=0.75pt,
  arc=1.8pt,
  left=6pt,right=6pt,top=4pt,bottom=4pt,
  before skip=6pt,after skip=7pt,
  fontupper=\small
}

\newif\ifdraftmarks
\draftmarksfalse
\definecolor{draftpending}{HTML}{7C3AED}
\definecolor{draftplaceholder}{HTML}{DC2626}
\definecolor{draftverify}{HTML}{D97706}

\definecolor{fbred}{HTML}{B91C1C}
\definecolor{fbrowhighlight}{HTML}{F7EEEE}

\draftmarkstrue

\newif\ifarxiv
\newcommand{\arxivcopy}{%
  \arxivtrue
  \iclrfinalcopy
  \pagestyle{fancy}
  \draftmarksfalse
}

\title{Train for Many, Update with One:\\
Improving Multi-Sample Reasoning via\\[-0.05em]
Fenchel--Bregman Posterior Learning}

\author{
\begin{tabular}{c}
Sungjun Lim$^{1,2}$\thanks{Work done during an internship at RIKEN AIP.}
\quad Adrian Robert Minut$^{1,3}$
\quad Nico Carlsen Daheim$^{4,5}$\\
\bfseries Mohammad Emtiyaz Khan$^{1,4,5}$
\quad Kyungwoo Song$^{2}$\thanks{Corresponding authors.}
\quad Thomas M{\"o}llenhoff$^{1}$\footnotemark[2]\\
\textnormal{$^{1}$ RIKEN AIP \quad $^{2}$ Yonsei University \quad $^{3}$ Sapienza University of Rome} \\
\textnormal{$^{4}$ TU Darmstadt \quad $^{5}$ Hessian.AI}
\end{tabular}
}

\arxivcopy 
\begin{document}

\ifarxiv
\begingroup
\centering
\maketitle\par
\endgroup
\lhead{Preprint.}
\else
\maketitle
\fi

\vspace{-1em}
\begin{abstract}
Multi-sample reasoning measures \textit{candidate coverage}, the probability that at least one of $K$ candidates is correct.
However, minimizing the average loss of individual attempts does not directly optimize candidate coverage.
We therefore formulate \textit{finite-$K$ posterior learning}, which trains a distribution over solver parameters to reduce the probability that all $K$ independently generated candidates fail.
To optimize this objective efficiently, we propose a Fenchel--Bregman (FB) approach that requires only one posterior draw per update, rather than the $K$ used by direct Monte Carlo optimization.
FB first uses an exact Fenchel reformulation to separate problem-specific weights from the posterior update, enabling one-draw posterior updates.
It then applies a closed-form Bregman update to track these weights through an exponential running average of observed failures, reusing the same draw.
We instantiate FB with IVON, a variational optimizer for Gaussian posteriors over solver parameters, yielding \textbf{Fenchel--Bregman IVON (FBI)}.
Across diverse recursive and autoregressive reasoning tasks, FBI improves candidate coverage.
In a matched comparison, it achieves comparable Pass@10 using 12.3\% of the direct Monte Carlo baseline's measured training time per update.
\end{abstract}

\section{Introduction}
\label{sec:introduction}

Complex reasoning problems often admit several solution paths, and a single attempt may fail even when another succeeds.
Multi-sample reasoning generates $K$ candidate solutions for the same problem~\citep{wang2023selfconsistency,cobbe2021verifiers}.
Its \textit{candidate coverage}, commonly measured by Pass@$K$, is the probability that at least one candidate is correct~\citep{chen2021codex,brown2024monkeys}.
For a fixed budget $K$, coverage depends on the distribution that generates the candidates.
This motivates learning the candidate-generating distribution rather than treating it as fixed.

We study this candidate-generation problem at the level of solver parameters.
Specifically, we learn a distribution $q_\phi(\theta)$ over solver parameters, where each draw instantiates a solver that produces one candidate.
Independent draws from $q_\phi$ form the $K$-candidate pool used at inference.
Recent methods align training with multi-sample inference by optimizing Pass@$K$ and other group-level objectives~\citep{pmlr-v267-tang25o,walder2025pkpo}.
Related approaches optimize multiple moments of problem-level failure~\citep{zhang2026beyond}, limit confidence to preserve repeated-generation coverage~\citep{chen2025rethinking}, or adapt fine-tuning to best-of-$N$ selection~\citep{chow2025inferenceaware}.
These methods train a single model to generate candidates for multi-sample inference, whereas we train a distribution over solver parameters for candidate coverage under a fixed budget $K$.

A natural way to train $q_\phi$, as in standard variational learning, is to minimize the expected loss of an individually sampled solver~\citep{blundell2015weight,shen2024variational}.
This treats each draw in isolation, while multi-sample reasoning succeeds when at least one of $K$ draws solves the problem.
We address this mismatch by training $q_\phi$ directly for the coverage of the $K$-candidate pool and call this formulation \textit{finite-$K$ posterior learning} (\Cref{sec:finite-k}).
Finite-$K$ posterior learning minimizes the probability that the entire $K$-draw candidate pool fails.
The resulting objective assigns a greater learning weight to problems that the other candidates are unlikely to solve.
We call the direct Monte Carlo estimator of this failure probability \textit{Direct-$K$}.
It is unbiased, but requires $K$ independent posterior draws per update.

To address this cost, we propose a \textbf{Fenchel--Bregman (FB)} optimization approach for finite-$K$ posterior learning using one posterior draw per update.
Using Fenchel duality, FB rewrites the objective so that the posterior can be updated using a weighted loss evaluated from one posterior draw.
This reformulation strategy is well established in compositional optimization and learning objectives~\citep{wang2017scgd,devraj2019stochastic,mohri2026distributional}, and we use it here to enable one-draw posterior updates for finite-$K$ candidate coverage.
As the posterior changes, a closed-form Bregman update adjusts the weights using past failure estimates and the current solver outcome (\Cref{sec:bregman-tracking}).
The same draw used for the posterior update also supplies the observation for the next weight update, avoiding additional posterior draws.
We instantiate FB with IVON~\citep{shen2024variational}, yielding \textbf{Fenchel--Bregman IVON (FBI)} (\Cref{sec:ivon-update}).
\Cref{fig:fbi_overview} summarizes this procedure.

\begin{figure*}[t]
    \centering
    \includegraphics[width=\textwidth]{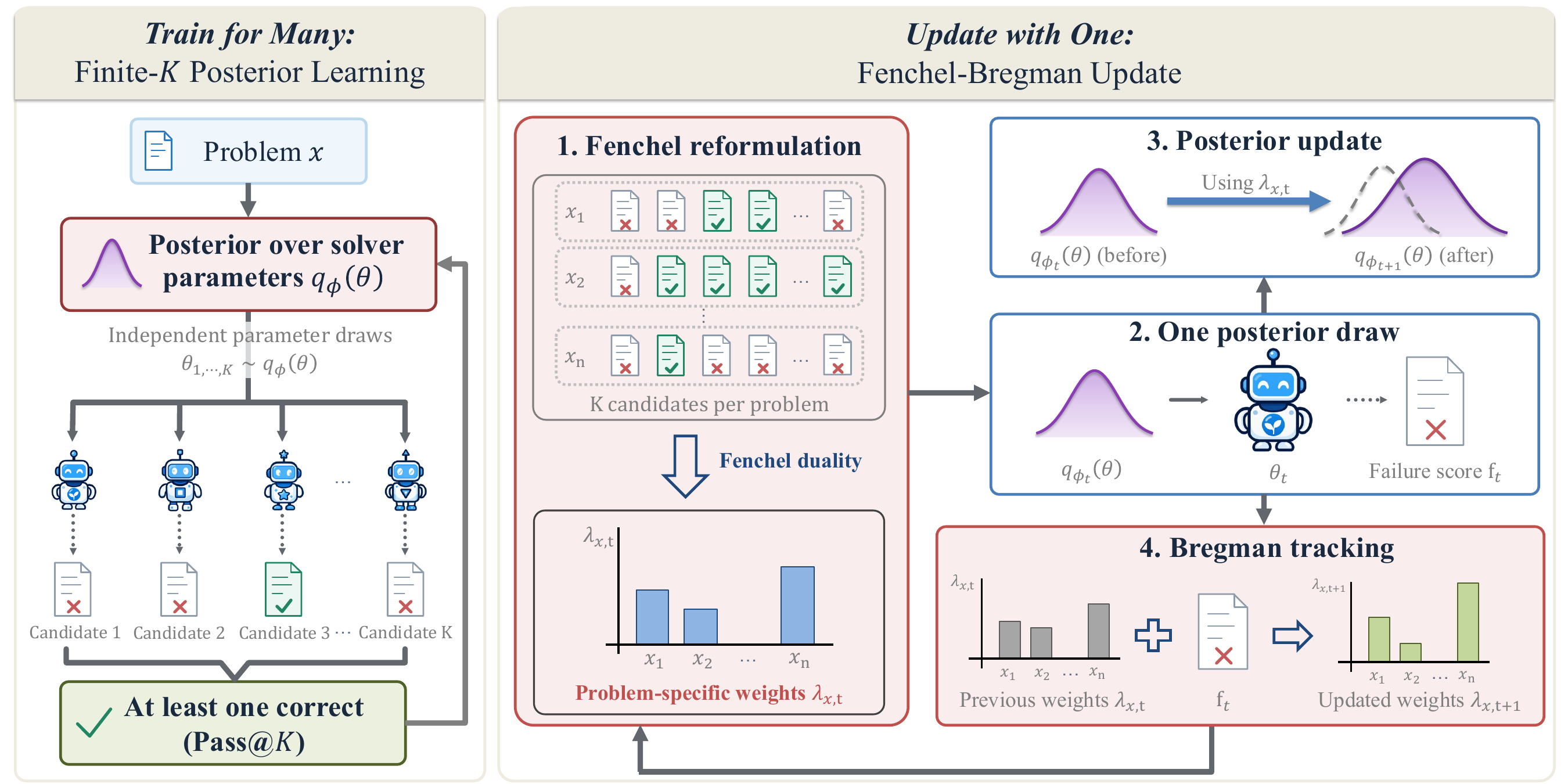}
    \caption{
    \textbf{Overview of finite-$K$ posterior learning and FBI.}
    Left: $K$ independent posterior draws generate a candidate pool, and training targets the probability that all $K$ candidates fail.
    Right: Fenchel reformulation enables a one-draw IVON posterior update with a fixed problem-specific weight, while Bregman tracking reuses the observed failure to update the weight for the next visit.
    }
    \label{fig:fbi_overview}
\end{figure*}

We evaluate FBI across recursive reasoning benchmarks with multiple solver backbones and mathematical reasoning with an autoregressive language model.
On a matched comparison, FB improves candidate coverage, with gains also observed across the evaluated backbones and diagonal and structured Gaussian posteriors.
In the same setting, FBI achieves comparable coverage to Direct-$K$ with substantially lower training computation.
Together, these results support aligning posterior training with multi-sample inference using only one new posterior draw per update.

To summarize, our contributions are the following:
\begin{itemize}[leftmargin=*]
    \item We formulate finite-$K$ posterior learning, aligning a posterior over solver parameters with the success of the finite candidate pool used at inference.
    \item We propose FB for one-draw posterior updates using an exact Fenchel reformulation and closed-form Bregman update for weight tracking, and instantiate FB with IVON as FBI.
    \item We show improved candidate coverage across recursive and autoregressive reasoning, and substantially lower training cost than Direct-$K$ in a matched comparison.
\end{itemize}

\section{\texorpdfstring{Train for Many: Finite-$K$ Posterior Learning}{Train for Many: Finite-K Posterior Learning}}
\label{sec:finite-k}

We now formalize the candidate-generating distribution underlying multi-sample reasoning and derive a training objective aligned with a finite candidate pool.
We first describe multi-sample reasoning under posterior parameter sampling in~\Cref{subsec:multi-sample_reasoning_with_posterior_draws}, then formulate finite-$K$ posterior learning in~\Cref{subsec:finite-k_posterior_learning}, and finally characterize the learning signal it induces in~\Cref{subsec:finite-k_learning_signal}.

\subsection{Multi-Sample Reasoning with Posterior Draws}
\label{subsec:multi-sample_reasoning_with_posterior_draws}
\paragraph{Multi-sample reasoning.}
Complex reasoning problems can admit multiple solution paths, and success may vary across attempts.
Multi-sample reasoning exploits this variation by generating $K$ candidate solutions for the same problem, usually from one solver model with parameters~$\theta$.
Candidate coverage, commonly measured by $\mathrm{Pass@}K$, is the probability that at least one of these candidates is correct.

\paragraph{Posterior candidate generation.}
For a fixed candidate budget $K$, coverage depends on how the candidate-generating distribution allocates probability across possible solutions.
This makes the distribution over candidates a natural object to optimize.
We represent this distribution through a learned distribution $q_\phi(\theta)$ over the parameters of a solver neural network. We then draw $K$ different solvers from this distribution and one candidate from each of these solvers. 
That is, the candidate pool used at inference consists of a total of $K$ outputs with one generated from each of the $K$ parameter draws $\theta_k\sim q_\phi$.
We seek to train $q_\phi$ for the collective success of this $K$-candidate pool.

\subsection{Finite-$K$ Posterior Learning}
\label{subsec:finite-k_posterior_learning}
\paragraph{From individual draws to candidate pools.}
When learning $q_\phi$, a natural objective is to minimize the expected loss of solvers sampled from this distribution, as for instance, a standard variational Bayesian approach would suggest.
However, this treats each draw individually, whereas multi-sample inference evaluates whether at least one of $K$ sampled solvers succeeds.
To align training with this inference criterion, we propose to optimize $q_\phi$ directly for $K$-candidate coverage.
We call this formulation \textit{finite-$K$ posterior learning}.

\paragraph{Finite-$K$ objective.}
For problem $x$, let $F(\theta,x)\in\{0,1\}$ denote the exact failure indicator for one draw, with $F=1$ for failure.
Write $\bar F_\phi(x)\coloneqq\mathbb E_{\theta\sim q_\phi}[F(\theta,x)]$ for the failure probability under $q_\phi$.
We draw $\theta_{1:K}$ independently from $q_\phi$.
Additional solver randomness is independent across candidates and included in all expectations. We omit it from the notation for readability.
Since the $K$ candidate draws are independent, the probability that they all fail is
\begin{equation}
  \Pr(\text{all $K$ candidates fail}\mid x)
  =
  \prod_{k=1}^K \Pr\!\left(F(\theta_k,x)=1\right)
  =
  \bar F_\phi(x)^K.
  \label{eq:binary-all-fail}
\end{equation}
Since candidate coverage is the complement of the all-fail event, averaging over the empirical distribution of training problems gives $\mathrm{Pass@}K(\phi)=1-\mathbb E_x[\bar F_\phi(x)^K]$.
Hence, maximizing coverage is equivalent to minimizing $\mathbb E_x[\bar F_\phi(x)^K]$.

Optimizing candidate coverage on the training problems alone may favor a highly concentrated parameter distribution.
We therefore add KL regularization toward a fixed reference prior $p_0(\theta)$, discouraging collapse toward a point estimate while controlling the spread of $q_\phi$.
This yields the proposed finite-$K$ posterior-learning objective:
\begin{equation}
\fcolorbox{lightgray}{white}{$
  \min_\phi~
  \mathbb E_x\!\left[\bar F_\phi(x)^K\right]
  +
  \tau_{\mathrm{KL}}D_{\mathrm{KL}}(q_\phi\,\Vert\,p_0)$.
  }
  \label{eq:exact-finite-k-objective}
\end{equation}
The first term minimizes the expected all-fail probability, while $\tau_{\mathrm{KL}}\geq0$ controls regularization toward the reference prior $p_0$.

\subsection{Understanding and Optimizing the Finite-$K$ Objective}
\label{subsec:finite-k_learning_signal}
\paragraph{Problem-dependent learning signal.}
To see how the objective in \Cref{eq:exact-finite-k-objective} differs from an individual-draw objective, we examine the per-problem gradient of its data term.
When $\bar F_\phi(x)$ is differentiable,
\begin{equation}
  \nabla_\phi \bar F_\phi(x)^K
  =
  \underbrace{K\bar F_\phi(x)^{K-1}}_{\lambda_x^\star}
  \nabla_\phi \bar F_\phi(x).
  \label{eq:finite-k-data-gradient}
\end{equation}
We denote this multiplicative factor by $\lambda_x^\star$, the problem-specific gradient weight induced by the objective in \Cref{eq:exact-finite-k-objective}.
For $K\geq2$, $\bar F_\phi(x)^{K-1}$ is the probability that the other $K-1$ candidates all fail.
Problems that the remaining candidates are unlikely to solve therefore receive a larger gradient weight, whereas problems already likely to be solved by another candidate receive a smaller one.
When $K=1$, $\lambda_x^\star=1$, so this problem-dependent reweighting disappears.

\paragraph{Naive Direct-$K$ estimation.}
For each problem $x$, $K$ independent posterior draws satisfy
\[
  \mathbb E_{\theta_{1:K}\sim q_\phi} \left[\prod_{k=1}^K F(\theta_k,x)\right]
  =
  \bar F_\phi(x)^K.
\]
Thus, $\prod_{k=1}^K F(\theta_k,x)$ is an unbiased estimator of the per-problem data term in \Cref{eq:exact-finite-k-objective}. We call this direct estimator \textit{Direct-$K$}.
Using Direct-$K$ to estimate the expected all-fail probability in \Cref{eq:exact-finite-k-objective} requires $K$ posterior draws per update, increasing the computational cost of each update.
\Cref{app:finite-k-details} provides the detailed derivation and the unbiasedness of this estimator.
The key computational challenge is to preserve the learning signal in \Cref{eq:finite-k-data-gradient} while using only a single new posterior draw per update.
In the following \Cref{sec:fb-optimization} we develop an efficient one-draw update for this purpose.

\section{Update with One: Fenchel--Bregman IVON}
\label{sec:fb-optimization}

In \Cref{sec:finite-k} we have defined the finite-$K$ learning signal and showed that Direct-$K$ estimates the data term in \Cref{eq:exact-finite-k-objective} using $K$ independent posterior draws per update.
Our goal is to retain this learning signal while using only one new posterior draw per update to improve training efficiency.
In the following \Cref{sec:fenchel} we will use a \textit{Fenchel reformulation} to separate the problem-specific weight from the posterior update, making a one-draw posterior update possible once the weight is fixed.
\Cref{sec:bregman-tracking} then uses \textit{Bregman tracking} to update this weight from the observed failure without additional posterior draws.
Finally, \Cref{sec:ivon-update} instantiates these updates with the variational learning algorithm IVON of \citet{shen2024variational} to optimize \Cref{eq:exact-finite-k-objective} for the distribution parameters $\phi$.
We call this IVON instantiation \textbf{Fenchel--Bregman IVON (FBI)}.

\subsection{Fenchel reformulation and one-draw posterior update}
\label{sec:fenchel}

\paragraph{Why direct estimation needs $K$ draws.}
Direct-$K$ uses $K$ independent posterior draws to estimate the gradient of $\bar F_\phi(x)^K$. Reusing a single draw does not preserve this quantity because, for binary $F$,
\[
  \mathbb E_{\theta\sim q_\phi}[F(\theta,x)^K]
  =
  \bar F_\phi(x),
\]
which generally differs from $\bar F_\phi(x)^K$. \Cref{eq:finite-k-data-gradient}, however, shows that the desired data gradient has the form $\lambda_x^\star\nabla_\phi\bar F_\phi(x)$.
This suggests the possibility of separating the problem-specific weight from the posterior update so that the posterior can be updated from a single draw once the weight is fixed.

\paragraph{Fenchel reformulation.}
Fenchel conjugacy~\citep{rockafellar1970convex} provides this separation by expressing the convex function $r^K$ as a maximization over weighted linear functions of $r\in[0,1]$, which we set to $r=\bar F_\phi(x)$. For $K\geq2$, we have the following representation, 
\begin{equation}
  r^K
  =
  \max_{\lambda\in[0,K]}
  \lambda r-\psi_K(\lambda),
  \qquad
  \lambda^\star(r)=Kr^{K-1},
  \label{eq:fenchel-identity}
\end{equation}
where $\psi_K(\lambda)=(K-1)(\lambda/K)^{K/(K-1)}$ defined on $[0,K]$ is the so-called Fenchel conjugate of $r^K$.
Such representations exist for any convex function~\citep{rockafellar1970convex} and for completeness, we recall in App.~\ref{app:fenchel-proof} that the identity is exact and that $\lambda^\star(r)$ is its unique maximizer.

Applying \Cref{eq:fenchel-identity} to $r=\bar F_\phi(x)$ in \Cref{eq:exact-finite-k-objective} introduces one weight $\lambda_x\in[0,K]$ for each training problem as a dual variable that acts as a problem-specific weight, giving the equivalent saddle-point objective
\begin{equation}
\fcolorbox{lightgray}{white}{$
  \min_\phi
  \max_{\lambda_x\in[0,K]}~
    \mathbb E_x
    \left[
      \mathbb E_{\theta\sim q_\phi}[\lambda_x F(\theta,x)]
      -\psi_K(\lambda_x)
    \right]
    +
    \tau_{\mathrm{KL}}D_{\mathrm{KL}}(q_\phi\, \Vert\,  p_0).
  $}
  \label{eq:saddle-objective}
\end{equation}
\Cref{eq:saddle-objective} separates $\phi$ and $\lambda_x$ into different optimization variables.
Holding $\lambda_x$ fixed therefore defines a posterior-update subproblem whose $\phi$-dependent data term is simply $\mathbb E_{\theta\sim q_\phi}[\lambda_xF(\theta,x)]$, while the KL term is unchanged.
The maximizing weight is $\lambda_x^\star=K\bar F_\phi(x)^{K-1}$, exactly the problem-specific gradient weight in \Cref{eq:finite-k-data-gradient}.
For $\tau_{\mathrm{KL}}>0$, this fixed-weight posterior subproblem has the generalized Bayesian variational form of \citet{bissiri2016general} which can be solved, for example, with the  IVON optimizer of \citet{shen2024variational}.

\paragraph{Fixed-weight one-draw posterior update.}
Since $\lambda_x$ is a separate variable in \Cref{eq:saddle-objective}, we can hold it fixed while taking a posterior step in $\phi$.
For fixed $\lambda_x$,
\[
  \nabla_\phi
  \mathbb E_{\theta\sim q_\phi}
  [\lambda_xF(\theta,x)]
  =
  \lambda_x\nabla_\phi\bar F_\phi(x),  
  \qquad
  \lambda_x=\lambda_x^\star \quad 
  \Longrightarrow \quad 
  \nabla_\phi\bar F_\phi(x)^K.
\]
Under standard conditions, one posterior draw gives an unbiased estimate of this fixed-weight gradient, and at $\lambda_x=\lambda_x^\star$ it recovers the finite-$K$ data gradient in \Cref{eq:finite-k-data-gradient}; see \Cref{app:fixed-dual-gradient}.
The posterior step can therefore use one draw.
The remaining problem is to obtain a suitable $\lambda_x$ without additional draws.

\subsection{Bregman update for the dual weight}
\label{sec:bregman-tracking}

\paragraph{Why the weight must be tracked.}
The optimal weight $\lambda_x^\star=K\bar F_\phi(x)^{K-1}$ changes as $q_\phi$ changes.
Re-estimating the current $\bar F_\phi(x)$ from multiple additional posterior draws at every update would remove the one-draw advantage.
At update step $t$, let $\lambda_{x,t}$ denote the weight used for problem $x$ and let $F_t=F(\theta_t,x)$ denote the failure observed from the single draw $\theta_t$. We seek an update from $\lambda_{x,t}$ to $\lambda_{x,t+1}$ using $F_t$, rather than recomputing $\lambda_x^\star$ from additional draws.

\paragraph{Bregman tracking.}
To update the dual weight $\lambda_x$, we use a Bregman update generated by the convex function $\psi_K$ in the Fenchel representation~\citep{bregman1967relaxation}.
This Bregman update can equivalently be viewed as a mirror-ascent step with $\psi_K$ as the mirror potential.
\Cref{app:bregman-proof} gives the proximal objective and its derivation.
The resulting update reduces to a simple exponential running average of the observed failures:
\begin{equation}
\fcolorbox{lightgray}{white}{$
  \widehat F_{x,t+1}
  =
  \beta\widehat F_{x,t}
  +(1-\beta)F_t,
  \qquad
  \lambda_{x,t+1}
  =
  K\widehat F_{x,t+1}^{K-1},
  $}
  \label{eq:dual-ema}
\end{equation}
where $\beta\in(0,1)$ controls how strongly past failure estimates are retained in the exponential running average.
The updated running estimate $\widehat F_{x,t+1}$ aggregates observed failures, and $K\widehat F_{x,t+1}^{K-1}$ gives the next problem-specific weight.
The update is closed form and uses no additional posterior draw.

\subsection{Posterior updates with IVON}
\label{sec:ivon-update}

\paragraph{Failure score and gradient estimation.}
Practical training uses a bounded failure score $f(\theta,x)\in[0,1]$, with posterior mean $\bar f_\phi(x):=\mathbb{E}_{\theta\sim q_\phi}[f(\theta,x)]$.
The score may be a differentiable surrogate for binary failure or be estimated from sampled binary outcomes.
When $f$ differs from the binary failure indicator $F$, this yields a score-based surrogate of the exact finite-$K$ objective in \Cref{eq:exact-finite-k-objective}.
Accordingly, the fixed-weight posterior gradient can be obtained either by backpropagation through $f$ or by a stochastic gradient estimator.
In both cases, the resulting failure observations are also used to update the running estimate $\widehat f$.

\begin{algorithm}[ht]
  \caption{\textbf{Fenchel--Bregman IVON (FBI) for finite-$K$ posterior learning.}}
  \label{alg:finite-k}
  \small
  \begin{algorithmic}[1]
    \Require
    Training set $\mathcal D$, diagonal Gaussian posterior $q_{\phi_0}$, prior $p_0$,
    pool size $K\geq2$, KL coefficient $\tau_{\mathrm{KL}}$,
    EMA coefficient $\beta\in(0,1)$,
    failure score $f$ and initial failure estimate $\widehat f_{\mathrm{init}}\in[0,1]$
    \State
    Initialize the IVON posterior mean and diagonal precision from $\phi_0$, and $\widehat f_{x,0} \gets \widehat f_{\mathrm{init}}$ for every $x \in \mathcal D$
    \For{$t=0,1,2,\ldots$}
      \State Sample a minibatch $\mathcal B_t$ from $\mathcal D$
      \State \textcolor{fbred}{\textbf{Form weight:}} $\lambda_{x,t} \gets K \widehat f_{x,t}^{K-1}$
      for each $x\in\mathcal B_t$
      \State \textcolor{fbred}{\textbf{Draw once:}} Draw $\theta_t \sim q_{\phi_{t}}$ and evaluate $f(\theta_t,x)$ for each $x\in\mathcal B_t$
      \State \textcolor{fbred}{\textbf{Estimate gradient:}} Compute a gradient estimate of \Cref{eq:fbi-weighted-loss}
      \Comment{fixed $\lambda_{x,t}$}
      \State
      \textcolor{fbred}{\textbf{Update posterior:}}
      Apply IVON from $\phi_{t}$ to $\phi_{t+1}$ using the gradient estimate and KL regularization
      \State \textcolor{fbred}{\textbf{Update failure estimate:}}
      $\widehat f_{x,t+1}$
      $\gets$
      $\beta \widehat f_{x,t}
      +(1-\beta)f(\theta_t,x)$
      for each $x\in\mathcal B_t$
      \Comment{others unchanged}
    \EndFor
  \end{algorithmic}
\end{algorithm}

\paragraph{IVON posterior.}
We instantiate the posterior update with IVON, which learns a diagonal Gaussian posterior~\citep{shen2024variational}.
We write $q_\phi(\theta)=\mathcal N(\theta;\mathbf m,\operatorname{diag}(\boldsymbol\sigma^2))$, with mean $\mathbf m$ and diagonal precision $\boldsymbol\sigma^{-2}$, so IVON learns both the mean and the parameter-wise sampling scales.
Under our batch-mean loss convention, IVON's effective sample size (ESS) sets the KL coefficient as $\tau_{\mathrm{KL}}=1/\mathrm{ESS}$.
ESS and weight decay jointly determine the implied Gaussian prior, so we rescale them inversely when varying the KL strength to keep the prior fixed.

\paragraph{One-draw IVON update.}
At update $t$, FBI first forms $\lambda_{x,t}=K\widehat f_{x,t}^{K-1}$ and then draws one $\theta_t\sim q_{\phi_{t}}$.
For a minibatch $\mathcal B_t$, the fixed-weight data loss is given as follows:
\begin{equation}
  \mathcal L^{\mathrm{data}}_t(\theta_t)
  =
  \frac{1}{|\mathcal B_t|}
  \sum_{x\in\mathcal B_t}
  \lambda_{x,t}f(\theta_t,x).
  \label{eq:fbi-weighted-loss}
\end{equation}
IVON uses a gradient estimate of \Cref{eq:fbi-weighted-loss}, together with KL regularization to update the posterior mean and diagonal precision.
Since $\lambda_{x,t}$ is fixed during this posterior update, the Fenchel penalty $\psi_K(\lambda_{x,t})$ is constant and omitted from \Cref{eq:fbi-weighted-loss}.
The same failure observations then update $\widehat f_{x,t+1}$ through \Cref{eq:dual-ema}, so one posterior draw both updates $q_\phi$ and supplies the feedback used to prepare the next weight.

Overall, \Cref{alg:finite-k} follows a stochastic alternating primal--dual optimization for \Cref{eq:saddle-objective}.
With the dual weights fixed, IVON updates the posterior parameters $\phi$.
The same failure observations are then used in the Bregman update of the dual weights, which is equivalently a mirror-ascent step on \Cref{eq:saddle-objective}.

\section{Experiments}
\label{sec:experiments}
We evaluate whether FBI improves candidate coverage while reducing the training cost of finite-$K$ posterior learning.
$K$ denotes the number of candidates at evaluation (the $K$ in Pass@$K$), whereas $K_{\mathrm{train}}$ denotes the candidate count targeted by training.\footnote{Code is available at \url{https://github.com/team-approx-bayes/Fenchel-Bregman}.}
We first report benchmark results and candidate coverage across evaluation budgets for recursive reasoning models and an autoregressive language model (\Cref{sec:main-results,sec:autoregressive-results}).
We then compare performance across recursive backbones and Bayesian methods (\Cref{sec:backbone-bayesian-comparison}).
Finally, we analyze the training target, problem-specific weighting, failure recovery, and one-draw training cost (\Cref{sec:coverage-training-efficiency}).
Detailed protocols and hyperparameters are in \Cref{app:experimental-details}, and sensitivity analyses are in \Cref{app:ablations}.

\subsection{Recursive reasoning models}
\label{sec:main-results}

\paragraph{Setup and baselines.}
We evaluate FBI with recursive reasoning models on the Sudoku-Extreme, Maze-Hard, ARC-AGI-1, and ARC-AGI-2 tasks.
Sudoku-Extreme requires completing difficult $9\times9$ puzzles, while Maze-Hard requires finding optimal paths in $30\times30$ mazes.
The ARC-AGI benchmarks require inferring a grid transformation from a few examples and applying it to a new input~\citep{chollet2019measure,chollet2025arcagi2}.
We report candidate coverage and, where applicable, selected accuracy, which is defined as the exact-solution accuracy of each method's selector.

Our baselines include the Hierarchical Reasoning Model (HRM)~\citep{wang2025hierarchical}, Tiny Recursive Model (TRM)~\citep{jolicoeurmartineau2025trm}, Fixed-Point Reasoning Model (FPRM)~\citep{movahedi2026fixedpoint}, and Probabilistic Tiny Recursive Model (PTRM)~\citep{sghaier2026probabilistic}.
We also include the fixed-Gaussian weight-sampling controls W-PTRM and W-FPRM.
FBI uses TRM in the main benchmark comparison and FPRM in the additional backbone comparison.
Matched runs share the backbone, posterior initialization, training data, update budget, and evaluation budget.
For FBI, the Acc. and Pass@10 values in \Cref{tab:main-benchmark} and all results are means over three independent training runs.

\begin{table*}[ht]
  \caption{\textbf{Selected accuracy (Acc.) and candidate coverage (Pass@$K$) across four recursive reasoning benchmarks.}
  All values are percentages.
  \textbf{Bold} and \underline{underlined} numerical values mark the highest and second-highest values in each metric column, with ties marked equally.
  \textsuperscript{$\dagger$} denotes previously reported performance.
  FBI achieves or ties the highest value in every metric column.}
  \label{tab:main-benchmark}
  \centering
  \scriptsize
  \setlength{\tabcolsep}{3pt}
  \begin{adjustbox}{max width=1.0\linewidth}
  \begin{tabular}{l|cccc|cccccc}
    \toprule
    & \multicolumn{2}{c}{Sudoku-Extreme} & \multicolumn{2}{c|}{Maze-Hard} & \multicolumn{3}{c}{ARC-AGI-1} & \multicolumn{3}{c}{ARC-AGI-2} \\
    \cmidrule(lr){2-3}\cmidrule(lr){4-5}\cmidrule(lr){6-8}\cmidrule(lr){9-11}
    Method & Acc. $\uparrow$ & Pass@10 $\uparrow$ & Acc. $\uparrow$ & Pass@10 $\uparrow$ & Pass@1 $\uparrow$ & Pass@2 $\uparrow$ & Pass@100 $\uparrow$ & Pass@1 $\uparrow$ & Pass@2 $\uparrow$ & Pass@100 $\uparrow$ \\
    \midrule
    HRM\textsuperscript{$\dagger$} & 55.00 & -- & 74.50 & -- & -- & 40.30 & -- & -- & 5.00 & -- \\
    FPRM & 94.33 & -- & 86.60 & -- & -- & \textbf{47.50} & -- & -- & 6.20 & -- \\
    TRM & 87.28 & -- & 87.00 & -- & 40.00 & 44.38 & 59.50 & \underline{5.00} & \underline{7.50} & 11.53 \\
    PTRM & 97.36 & 97.38 & \underline{87.70} & 88.40 & \underline{40.50} & 44.88 & 60.25 & \underline{5.00} & \underline{7.50} & \underline{11.94} \\
    W-PTRM & \underline{98.83} & \underline{98.84} & 87.50 & \underline{89.80} & 40.37 & \underline{45.25} & \underline{61.25} & \underline{5.00} & 7.08 & 10.69 \\
    \rowcolor{fbrowhighlight}
    \textbf{FBI (Ours)} & \textbf{99.53} & \textbf{99.55} & \textbf{88.00} & \textbf{91.30} & \textbf{40.63} & \textbf{47.50} & \textbf{64.25} & \textbf{5.42} & \textbf{8.19} & \textbf{12.78} \\
    \bottomrule
  \end{tabular}
  \end{adjustbox}
\end{table*}

\paragraph{Results.}
FBI improves candidate coverage and selected accuracy on Sudoku-Extreme and Maze-Hard (\Cref{tab:main-benchmark}).
On Sudoku-Extreme, FBI nearly matches its Pass@10 of 99.55\% with 99.53\% selected accuracy.
On Maze-Hard, FBI raises Pass@10 from W-PTRM's 89.80\% to 91.30\%, while reaching 88.00\% selected accuracy.
The remaining gap on Maze-Hard shows that some correct candidates are still not returned by the selector.
FBI also improves candidate coverage after aggregation and ranking on ARC-AGI-1 and ARC-AGI-2.
It achieves or ties the highest Pass@$K$ across the evaluated values of $K$.

\paragraph{Candidate coverage.}
We examine coverage across candidate budgets in recursive models.
Here, the coverage panels compare equal numbers of candidate attempts instead of equal wall-clock times or FLOP budgets across models.
We compare FBI with Direct-10, PTRM, and W-PTRM to assess coverage under fixed evaluation budgets.
FBI and Direct-10 both target $K_{\mathrm{train}}=10$, but use one and ten posterior draws per training update, respectively.
We include Direct-10 to compare FBI's candidate coverage and training cost against a direct estimator of the same finite-$K_{\mathrm{train}}$ target.
At test time, all four methods generate multiple candidates with the same TRM backbone and the same reasoning depth of 16.
\Cref{fig:maze-hard-pass-at-k} compares coverage at $K\in\{1,2,4,6,8,10\}$. FBI reaches the Pass@10 of PTRM, W-PTRM, and Direct-10 using 80\%, 60\%, and 20\% fewer candidates, respectively.
Across the budgets, the reported FBI run is 0.2--0.5 percentage points above Direct-10.

\begin{figure}[ht]
  \centering
  \begin{subfigure}[t]{0.35\textwidth}
    \centering
    \vspace{-0.5em}
    \includegraphics[width=\linewidth]{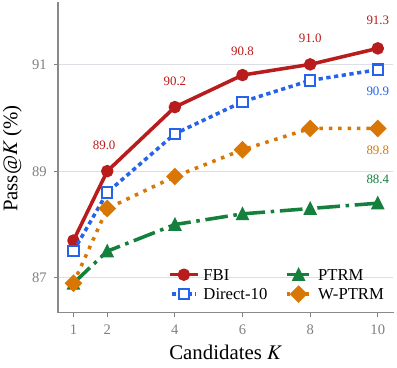}
    \caption{\textbf{Maze-Hard.}}
    \label{fig:maze-hard-pass-at-k}
  \end{subfigure}
  \hspace{0.1\textwidth}
  \begin{subfigure}[t]{0.35\textwidth}
    \centering
    \vspace{-0.5em}
    \includegraphics[width=\linewidth]{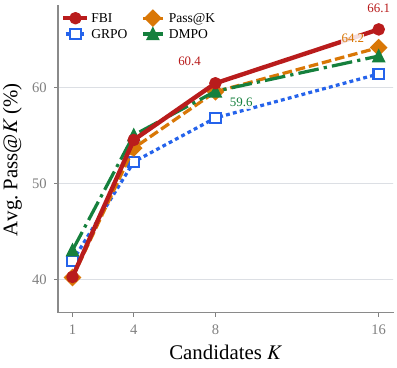}%
    \caption{\textbf{Mathematical reasoning.}}
    \label{fig:autoregressive-pass-at-k}
  \end{subfigure}
  \caption{\textbf{Candidate coverage across recursive and autoregressive models.}
  (a) Pass@$K$ on 1,000 Maze-Hard test problems using TRM at depth 16.
  (b) Mean Pass@$K$ across five mathematical-reasoning benchmarks using Qwen3-4B-Base.
  FBI improves candidate coverage in both settings.}
  \label{fig:candidate-coverage}
\end{figure}

\vspace{-1em}
\subsection{Autoregressive language models}
\label{sec:autoregressive-results}

\paragraph{Setup and baselines.}
We use Qwen3-4B-Base to test FBI beyond recursive solvers.
Following the mathematical-reasoning setup of MMPO~\citep{zhang2026beyond}, we train on MATH7.5K~\citep{hendrycks2021math} and evaluate on MATH500~\citep{lightman2024verify}, OlymMATH~\citep{sun2026olymmath}, AMC23~\citep{maa2023amc12}, AIME24~\citep{maa2024aime}, and AIME25~\citep{maa2025aime}.
Each FBI candidate uses an independent parameter draw from the learned posterior, followed by autoregressive sampling.
FBI places a diagonal IVON posterior over all actor parameters and trains with $K_{\mathrm{train}}=8$.
Each optimizer update uses one posterior parameter draw shared across eight responses per prompt.
\Cref{app:autoregressive-details} provides the implementation details.

We compare FBI with Base and five baselines.
GRPO~\citep{shao2024deepseekmath} maximizes mean reward with group-relative advantages.
Pass@$K$ Training~\citep{chen2025pass} optimizes group success using analytic response-level advantages.
DMPO~\citep{li2026beyondmodecollapse} matches a reward-proportional group distribution to preserve diversity.
MaxRL~\citep{tajwar2026maximum} uses compute-indexed objectives approaching maximum likelihood as the budget grows.
MMPO~\citep{zhang2026beyond} minimizes failure-probability moments to target truncated time to first success.

\begin{table*}[ht]
  \caption{\textbf{Mathematical reasoning performance with Qwen3-4B-Base.}
  Columns report Pass@8 (\%); \textbf{Avg.} is the unweighted mean across five benchmarks.
  \textbf{Bold} and \underline{underlined} values mark the best and second-best results per column.
  FBI achieves the highest average Pass@8 (60.44\%), 0.84 percentage points above DMPO.}
  \label{tab:autoregressive-standard}
  \vspace{-0.5em}
  \centering
  \adjustbox{max width=0.85\linewidth}{
  \begin{tabular}{l|ccccc|c}
    \toprule
    Method & MATH500 & OlymMATH & AMC23 & AIME24 & AIME25 & \textbf{Avg.} \\
    \midrule
    Base & 87.80 & 53.57 & 75.18 & 34.15 & 21.56 & 54.45 \\
    GRPO & 90.17 & 61.96 & 79.07 & 27.96 & 24.95 & 56.82 \\
    Pass@$K$ Training & \textbf{92.48} & \underline{62.24} & 82.64 & \underline{34.64} & 25.76 & 59.55 \\
    DMPO & 91.37 & 61.45 & \underline{84.68} & \textbf{34.65} & 25.84 & \underline{59.60} \\
    MaxRL & 91.23 & \textbf{62.26} & 83.07 & 28.08 & \underline{30.21} & 58.97 \\
    MMPO & 89.91 & 58.99 & 81.87 & 26.89 & 25.17 & 56.56 \\
    \rowcolor{fbrowhighlight}
    \textbf{FBI (Ours)} & \underline{91.64} & 60.68 & \textbf{85.20} & 31.00 & \textbf{33.70} & \textbf{60.44} \\
    \bottomrule
  \end{tabular}
  }
\vspace{-0.5em}
\end{table*}

\paragraph{Mathematical reasoning performance.}
FBI delivers the strongest aggregate Pass@8 performance in \Cref{tab:autoregressive-standard}, with an unweighted average of 60.44\%.
Beyond the aggregate result, FBI achieves the best scores on AMC23 and AIME25.

\vspace{-0.5em}
\paragraph{Candidate coverage.}
\Cref{fig:autoregressive-pass-at-k} compares FBI to GRPO, Pass@$K$ Training, and DMPO using average Pass@$K$ across the same five benchmarks.
At $K=1$, FBI reaches 40.22\%, below DMPO at 43.01\%.
The gap narrows at $K=4$, and FBI overtakes all three baselines at $K=8$ with 60.44\%, 0.84 percentage points above DMPO.
At $K=16$, FBI reaches 66.06\%, 1.90 percentage points over Pass@$K$ Training.
FBI's stronger performance with multiple candidates is consistent with its objective of improving candidate coverage.

\subsection{Performance across backbones and Bayesian methods}
\label{sec:backbone-bayesian-comparison}

We compare recursive backbones and Bayesian methods on Maze-Hard to assess whether FB's coverage gains extend beyond TRM and IVON's diagonal Gaussian posterior.
FB improves candidate coverage on both evaluated recursive backbones and with a structured Gaussian posterior on TRM.
Both FB variants also achieve higher Pass@10 than all evaluated Bayesian baselines.

\begin{figure}[ht]
  \centering
  \begin{subfigure}[t]{0.49\textwidth}
    \centering
    \includegraphics[width=\linewidth]{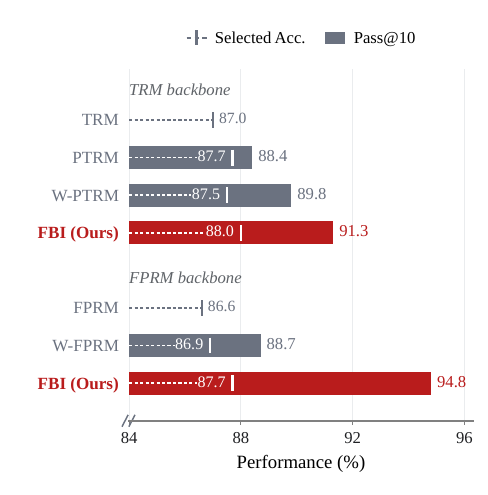}
    \caption{\textbf{Recursive backbones.}}
    \label{fig:maze-cross-backbone}
  \end{subfigure}\hfill%
  \begin{subfigure}[t]{0.49\textwidth}
    \centering
    \includegraphics[width=\linewidth]{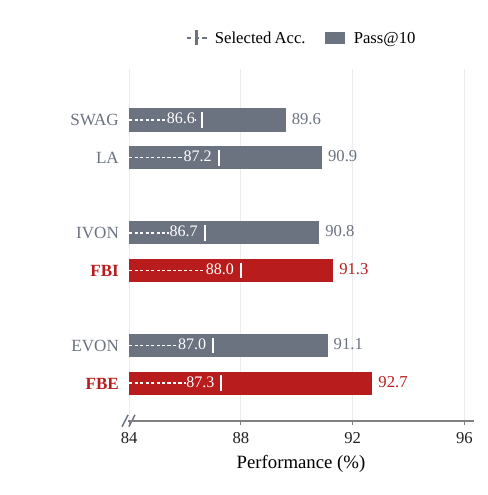}
    \caption{\textbf{Bayesian methods.}}
    \label{fig:maze-bayesian-methods}
  \end{subfigure}
  \caption{\textbf{Performance across backbones and Bayesian methods.}
  (a) Comparison across TRM and FPRM backbones.
  (b) Comparison of Bayesian methods with the TRM backbone.
  Bars show Pass@10, and dashed markers show selected accuracy.
  FB-based methods achieve higher Pass@10 than the corresponding baselines across the evaluated recursive backbones and variational optimizers.}
  \label{fig:maze-backbone-bayesian-comparison}
\end{figure}

\paragraph{Performance across recursive backbones.}
Figure~\ref{fig:maze-cross-backbone} tests whether the Maze-Hard gains extend from TRM to FPRM, using ten candidates and each method's selector.
With FPRM, FBI reaches 87.70\% selected accuracy and 94.80\% Pass@10, compared with 86.90\% and 88.70\% for W-FPRM.
Together with the TRM results, this shows that the gains over fixed-Gaussian weight sampling obtained using a learned posterior are not limited to one recursive backbone.

\paragraph{Performance across Bayesian methods.}
We compare stochastic weight averaging Gaussian (SWAG)~\citep{maddox2019swag}, Laplace approximation (LA)~\citep{daxberger2021laplace}, IVON, and Eigenspace-VON (EVON)~\citep{minut2026soap}.
All methods use TRM on Maze-Hard with ten candidates at evaluation.
We apply FB with EVON's structured Gaussian posterior, denoting this variant Fenchel--Bregman EVON (FBE).
FBI and FBE both target $K_{\mathrm{train}}=10$.
Appendix~\ref{app:bayesian-comparison} specifies the posterior constructions.

FBI and FBE achieve higher candidate coverage than IVON and EVON, respectively (\Cref{fig:maze-bayesian-methods}).
FBI raises Pass@10 from IVON's 90.8\% to 91.3\% and selected accuracy from 86.7\% to 88.0\%.
FBE reaches 92.7\% Pass@10 and 87.3\% selected accuracy, compared with 91.1\% and 87.0\% for EVON.
FBE's higher selected accuracy reflects increased candidate coverage, despite a lower fraction of correct selections among problems with a correct candidate.
Both FB variants exceed LA and SWAG in Pass@10, supporting posterior learning that explicitly targets finite-$K$ candidate coverage.
FBE leads in candidate coverage and FBI in selected accuracy among these methods.

\subsection{Analysis of finite-$K$ posterior learning}
\label{sec:coverage-training-efficiency}

We first examine how the training target and problem-specific weights affect candidate coverage on Maze-Hard.
We then analyze failure recovery and one-draw training cost.

\begin{figure}[ht]
  \centering
  \begin{subfigure}[t]{0.32\textwidth}
    \centering
    \includegraphics[width=\linewidth]{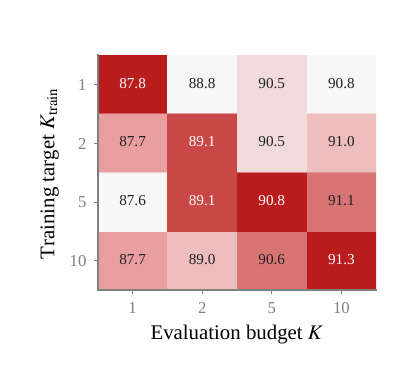}
    \caption{\textbf{Training target $K_{\mathrm{train}}$.}}
    \label{fig:maze-hard-training-target}
  \end{subfigure}\hfill%
  \begin{subfigure}[t]{0.32\textwidth}
    \centering
    \includegraphics[width=\linewidth]{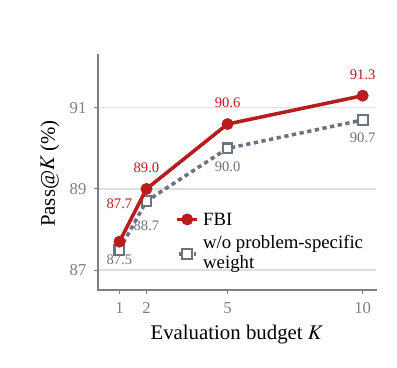}%
    \caption{\textbf{Problem-specific weights.}}
    \label{fig:maze-hard-problem-weights}
  \end{subfigure}\hfill%
  \begin{subfigure}[t]{0.32\textwidth}
    \centering
    \includegraphics[width=\linewidth]{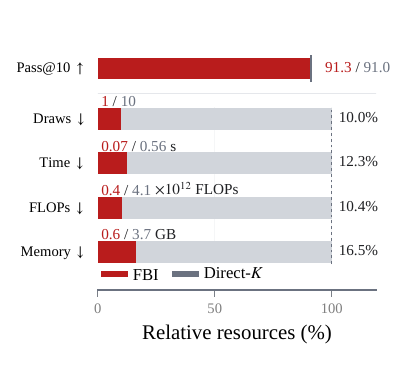}
    \caption{\textbf{Training cost.}}
    \label{fig:maze-hard-training-efficiency}
  \end{subfigure}
  \caption{\textbf{Finite-$K$ posterior learning and training efficiency.}
  (a) The diagonal pattern shows that $K_{\mathrm{train}}=K$ attains the highest reported Pass@$K$ at each budget.
  Cells show Pass@$K$ (\%); darker shades indicate higher within-column coverage.
  (b) Problem-specific weights outperform their minibatch-mean replacement at $K_{\mathrm{train}}=10$.
  (c) At $K_{\mathrm{train}}=10$, FBI achieves comparable coverage to Direct-10 using 12.3\% of its measured training time per update.
  Resources are relative to Direct-10 (100\%); FLOPs are in units of $10^{12}$.
  }
  \label{fig:maze-hard-finite-k-analysis}
\end{figure}

\vspace{-0.5em}
\paragraph{Effect of $K_{\mathrm{train}}$ on candidate coverage.}
We examine how the training target $K_{\mathrm{train}}\in\{1,2,5,10\}$ affects Pass@$K$ across evaluation budgets on Maze-Hard, retaining one posterior draw per training update (\Cref{fig:maze-hard-training-target}).
The columnwise rankings exhibit a diagonal pattern: matching $K_{\mathrm{train}}$ to the evaluation budget $K$ gives the highest reported Pass@$K$ at every tested budget.
At $K=2$, the matched target ties with $K_{\mathrm{train}}=5$ at 89.1\%.
Increasing $K_{\mathrm{train}}$ from 1 to 10 raises Pass@10 from 90.8\% to 91.3\%, while Pass@1 changes from 87.8\% to 87.7\%.
This pattern supports choosing $K_{\mathrm{train}}$ according to the intended evaluation budget and shows that the training target affects candidate-pool coverage with little change in single-candidate accuracy.

\paragraph{Role of problem-specific weighting.}
We test the contribution of problem-specific weighting at $K_{\mathrm{train}}=10$ by replacing each problem's weight with the minibatch mean, removing differences between problems while preserving the mean weight scale (\Cref{fig:maze-hard-problem-weights}).
Removing these differences lowers coverage at every tested budget: Pass@1 falls from 87.7\% to 87.5\%, and Pass@10 from 91.3\% to 90.7\%.
FBI's gain grows from 0.2 percentage points at $K=1$ to 0.6 points at $K=5$ and 10.
These results support a contribution from problem-specific weighting to candidate coverage, with larger gains when evaluating multiple candidates.

\paragraph{Recovery after first-candidate failures.}
\label{sec:first-failure-recovery}

We analyze recovery among candidates 2--10 after each method's first-candidate failures on one fixed Maze-Hard seed (Figure~\ref{fig:maze-hard-first-failure-recovery}).
FBI recovers an exact solution in 32.5\% of these cases (40/123), compared with 27.2\% for Direct-10, 22.1\% for W-PTRM, and 11.5\% for PTRM.
This recovery raises FBI's coverage from 87.7\% Pass@1 to 91.7\% Pass@10.
FBI also has a lower repeated-output rate (56.9\% versus 64.0--87.8\% for the baselines), so later candidates more often change the output and recover an exact solution.
The weight ablation (Figure~\ref{fig:maze-hard-problem-weights}) supports problem-specific weighting's contribution to coverage, but does not isolate its effect on conditional recovery.

\begin{figure*}[ht]
  \centering
\includegraphics[width=\textwidth]{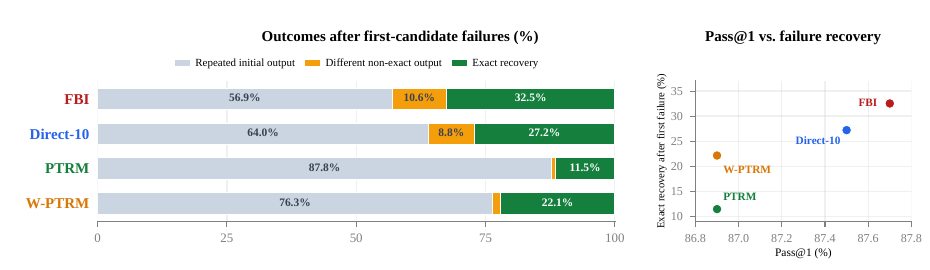}
\vspace{-2em}
  \caption{\textbf{Recovery after first-candidate failures on Maze-Hard.}
  Left: first-candidate failures are partitioned into three outcomes: only repeated initial outputs, different outputs with no exact solution, or at least one exact solution.
  Right: Pass@1 versus the exact-recovery rate among candidates 2--10 after each method's first candidate fails.
  FBI has the highest conditional exact-recovery rate of 32.5\%.}
  \label{fig:maze-hard-first-failure-recovery}
\end{figure*}

\vspace{-0.5em}
\paragraph{One-draw training cost.}
FBI achieves comparable mean Maze-Hard Pass@10 to Direct-10 across repeated runs at $K_{\mathrm{train}}=10$.
\Cref{fig:maze-hard-training-efficiency} reports these mean Pass@10 values alongside separately measured one-GPU training resources per update.
FBI uses one posterior draw per update, compared with ten for Direct-10.
FBI uses 12.3\% of Direct-10's measured training time and 10.4\% of its FLOPs per update.
Peak allocated GPU memory is 0.611\,GB for FBI and 3.707\,GB for Direct-10.
\Cref{app:compute-accounting} gives the one-GPU measurement details.

\section{Conclusion}
\label{sec:conclusion}

Fenchel--Bregman IVON (FBI) uses an exact reformulation and Bregman tracking for one-draw finite-$K$ posterior learning, improving candidate coverage across recursive and autoregressive reasoning.
On matched Maze-Hard, FBI reaches comparable Pass@10 to Direct-10 using 10.0\% as many posterior draws and 12.3\% of the measured training time per update.
On Qwen3-4B-Base, FBI attains the highest mean Pass@8 across five mathematical-reasoning benchmarks.
We also show that FB's gains across backbones and with a structured Gaussian posterior support its applicability beyond IVON.
However, surrogate mismatch, stochastic gradient noise, or posterior drift may weaken failure tracking.
Improving surrogate calibration where applicable and using predictive or variance-reduced tracking could make that signal more reliable.
Adaptive or correlated candidate pools and selector-aware training could better match inference, while amortized estimates for unseen problems and broader evaluation across tasks and posterior families could extend applicability.

\ifarxiv\else\clearpage\fi
\section*{AI Use Statement}
We used OpenAI GPT and Anthropic Claude to improve the clarity and readability of the manuscript, refine research ideas, and assist in developing mathematical derivations and proofs.
The authors reviewed all AI-assisted material and checked the mathematical arguments step by step, including their assumptions and derivations.
The authors take full responsibility for the final content of this work, including all AI-assisted text and mathematical claims.

\bibliographystyle{iclr2027_conference}
\bibliography{main}

\clearpage
\appendix

\section{\texorpdfstring{Derivations for finite-$K$ posterior learning and Fenchel--Bregman updates}{Derivations for finite-K posterior learning and Fenchel--Bregman updates}}
\label{app:proofs}

This appendix provides the derivations underlying the finite-$K$ objective and the one-draw Fenchel--Bregman update.
\Cref{app:finite-k-details} derives the binary all-fail probability and its direct $K$-draw Monte Carlo estimator.
\Cref{app:fenchel-proof} establishes the Fenchel representation and its optimal problem-specific weight, and \Cref{app:fixed-dual-gradient} shows that fixing this weight makes the posterior-dependent gradient estimable from one draw.
Finally, \Cref{app:bregman-proof} derives the Bregman update for the dual weight and its closed-form EMA implementation.

\subsection{Binary all-fail probability and direct Monte Carlo estimation}
\label{app:finite-k-details}
Fix a problem $x$ and draw $\theta_{1:K}\overset{\mathrm{iid}}{\sim}q_\phi$.
Any solver randomness $\xi_k$ is included in each draw, independently across candidates, and suppressed in the notation.
Since $F(\theta,x)\in\{0,1\}$,
\begin{equation}
  \Pr(\text{all $K$ candidates fail}\mid x)
  =
  \mathbb E\!\left[\prod_{k=1}^K F(\theta_k,x)\right]
  =
  \bar F_\phi(x)^K.
  \label{eq:app-binary-finite-k}
\end{equation}
The first equality uses the indicator property, and the second uses independence.
Thus, the product is an unbiased Direct-$K$ estimator using $K$ posterior draws, and candidate coverage is $1-\bar F_\phi(x)^K$.
Averaging the all-fail probability over training problems gives the data term in \Cref{eq:exact-finite-k-objective}.

For correlated draws, the product $\prod_{k=1}^K F(\theta_k,x)$ still equals one exactly when all $K$ candidates fail, so it remains the indicator of the all-fail event.
However, without independence its expectation does not generally factor into a product of marginal failure probabilities, and hence need not equal $\bar F_\phi(x)^K$.

Likewise, using the same posterior draw in all $K$ factors does not recover the finite-$K$ objective.
Because $F(\theta,x)\in\{0,1\}$,
\[
F(\theta,x)^K = F(\theta,x),
\]
so
\[
\mathbb E_{\theta\sim q_\phi}[F(\theta,x)^K]
=
\bar F_\phi(x),
\]
rather than $\bar F_\phi(x)^K$ in general.
Thus, the $K$-th-power objective relies on independent candidate draws, and a naive reuse of one draw cannot provide the same Monte Carlo estimator.

\subsection{Derivation of the Fenchel representation}
\label{app:fenchel-proof}
This section proves the Fenchel representation in \Cref{eq:fenchel-identity} and identifies its unique maximizer.
For $K\geq2$, let $\psi_K(\lambda)=(K-1)(\lambda/K)^{K/(K-1)}$ on $[0,K]$.
Its derivative and inverse are
\begin{equation}
  \psi_K'(\lambda)
  =\left(\frac{\lambda}{K}\right)^{1/(K-1)},
  \qquad
  (\psi_K')^{-1}(u)=Ku^{K-1},
  \quad u\in[0,1].
  \label{eq:app-dual-coordinate-derivative}
\end{equation}
The derivative is a strictly increasing bijection from $[0,K]$ to $[0,1]$, so $\psi_K$ is strictly convex.
For $r\in[0,1]$, the unique maximizer of the strictly concave function $\lambda r-\psi_K(\lambda)$ is $(\psi_K')^{-1}(r)$, including the endpoints $r=0,1$.
Hence,
\begin{equation}
  \begin{aligned}
    \lambda^\star(r)&=Kr^{K-1},\\
    \lambda^\star(r)r-\psi_K(\lambda^\star(r))
    &=Kr^K-(K-1)r^K=r^K.
  \end{aligned}
  \label{eq:app-fenchel-first-order-solution}
\end{equation}
This proves \Cref{eq:fenchel-identity}.
Setting $r=\bar F_\phi(x)$ gives \Cref{eq:saddle-objective} with the KL term unchanged and recovers the weight in \Cref{eq:finite-k-data-gradient}.
For $K=1$, the data term is linear and no Fenchel step is needed.

\subsection{One-draw gradient with a fixed weight}
\label{app:fixed-dual-gradient}
Fix a problem $x$ and suppose $q_\phi$ has a differentiable density whose support does not depend on $\phi$.
Assume differentiation can be exchanged with integration. The binary indicator $F(\theta,x)$ itself need not be differentiable.
The score-function identity gives
\begin{equation}
  \nabla_\phi\bar F_\phi(x)
  =
  \mathbb E_{\theta\sim q_\phi}\!\left[F(\theta,x)\nabla_\phi\log q_\phi(\theta)\right].
  \label{eq:app-fixed-dual-gradient}
\end{equation}
Let $\lambda_x$ be fixed with respect to $\phi$ and chosen before drawing $\theta$.
Multiplying the one-draw estimator in \Cref{eq:app-fixed-dual-gradient} by $\lambda_x$ gives an unbiased estimate of $\lambda_x\nabla_\phi\bar F_\phi(x)$.
When $\lambda_x=\lambda_x^\star$, this estimator is unbiased for the finite-$K$ data gradient in \Cref{eq:finite-k-data-gradient}.
At other fixed weights, it remains unbiased for the fixed-weight data gradient, but not generally for the finite-$K$ data gradient.
The KL term is outside this Monte Carlo argument and remains unchanged in the fixed-weight posterior subproblem.

At update $t$, $\lambda_{x,t}$ is determined before $\theta_t\sim q_{\phi_t}$ is drawn, from the current running estimate.
Conditioning on the information available before the draw therefore makes the weight fixed, so the same unbiasedness argument applies.
The observed failure $F_t$ then updates the running estimate for the next visit.

\subsection{Bregman update for the dual weight}
\label{app:bregman-proof}

For $K\geq2$, define the Bregman divergence generated by $\psi_K$ as
\[
D_{\psi_K}(\lambda,\mu)
=
\psi_K(\lambda)-\psi_K(\mu)-\psi_K'(\mu)(\lambda-\mu).
\]
Given the current weight $\lambda_{x,t}\in[0,K]$, the observed failure $F_t\in[0,1]$, and $\eta>0$, we consider the Bregman-proximal update
\begin{equation}
  \lambda_{x,t+1}
  =
  \arg\max_{\lambda\in[0,K]}
  \left\{
    \lambda F_t-\psi_K(\lambda)
    -\eta^{-1}D_{\psi_K}(\lambda,\lambda_{x,t})
  \right\}.
  \label{eq:app-bregman-prox}
\end{equation}
The first two terms are the per-step dual objective induced by the Fenchel representation, while the Bregman term regularizes the update toward the previous weight.

\begin{proposition}[Closed-form Bregman update]
\label{prop:bregman-update}
Let $K\geq2$, $\eta>0$, and suppose
\[
\lambda_{x,t} = K\widehat F_{x,t}^{K-1},
\qquad 
\widehat F_{x,t}\in[0,1].
\]
With $\beta=1/(1+\eta)$, the unique solution of
\Cref{eq:app-bregman-prox} is
\begin{equation*}
  \widehat F_{x,t+1} = \beta\widehat F_{x,t} +(1-\beta)F_t,
  \qquad
  \lambda_{x,t+1} = K\widehat F_{x,t+1}^{K-1},
\end{equation*}
which gives the update in \Cref{eq:dual-ema}.
\end{proposition}

\begin{proof}
From \Cref{app:fenchel-proof},
\[
\psi_K'(\lambda) = \left(\frac{\lambda}{K}\right)^{1/(K-1)}, \qquad (\psi_K')^{-1}(u) = Ku^{K-1}.
\]
Differentiating the objective in \Cref{eq:app-bregman-prox} gives the stationarity condition
\[
(1+\eta)\psi_K'(\lambda_{x,t+1})
=
\psi_K'(\lambda_{x,t})+\eta F_t.
\]
Since
$\lambda_{x,t}=K\widehat F_{x,t}^{K-1}$,
we have
$\psi_K'(\lambda_{x,t})=\widehat F_{x,t}$.
Therefore,
\begin{align}
\psi_K'(\lambda_{x,t+1})
&= \frac{\widehat F_{x,t}+\eta F_t}{1+\eta} \notag\\
&= \beta\widehat F_{x,t} +(1-\beta)F_t \coloneqq \widehat F_{x,t+1}.
\label{eq:app-bregman-coordinate}
\end{align}
Since $\widehat F_{x,t},F_t\in[0,1]$, their convex combination $\widehat F_{x,t+1}$ also lies in $[0,1]$.
Applying $(\psi_K')^{-1}$ then gives
\[
\lambda_{x,t+1}
=
K\widehat F_{x,t+1}^{K-1}
\in[0,K].
\]
Finally, the objective in \Cref{eq:app-bregman-prox} is strictly concave in $\lambda$, so this solution is the unique maximizer.
\end{proof}

Thus, the exponential running average in \Cref{eq:dual-ema} is not an independently chosen heuristic; it is the closed-form solution of the Bregman-proximal weight update in \Cref{eq:app-bregman-prox}.
\section{Experimental details}
\label{app:experimental-details}

In this section, we describe the evaluation protocols, baseline controls, and posterior-training settings used in \Cref{sec:experiments}.
We then specify the autoregressive evaluation and the training-cost measurements.
We use the failure score $f$ and running estimate $\widehat f$ defined in \Cref{sec:ivon-update}. \Cref{app:posterior-training,app:autoregressive-details} specify how the failure signal and its gradient estimate are instantiated in our experiments.

\subsection{Recursive evaluation protocols}
\label{app:recursive-evaluation}

\paragraph{Metrics and candidate budgets.}
Selected accuracy measures exact solutions returned by a method's selector, whereas Pass@$K$ measures exact-solution coverage among its first $K$ ordered candidates.
Here $K$ denotes the evaluation budget, while $K_{\mathrm{train}}$ denotes the training objective's target budget.
Posterior methods use independently sampled parameters, and stochastic solvers also resample their internal randomness.
Unless stated otherwise, $K=10$, and smaller budgets use ordered prefixes of the same candidate bank.
For ARC-AGI, these prefixes follow the aggregation and ranking below.
We randomly hold out 10\% of the original training problems for validation and use it to select noise scales, the selection margin, posterior scales, and checkpoints.

\paragraph{Sudoku-Extreme.}
Sudoku-Extreme expands 1,000 training source puzzles with digit permutations, valid row and column permutations, and optional transposition, yielding 1,001,000 rows.
The failure tracker is keyed by source-puzzle ID, so all variants of a source share one $\widehat f_x$ state.
When a minibatch contains repeated sources, we average their observed failure scores before applying one EMA update per source.
We evaluate candidates at reasoning depth 64 on all 422,786 test puzzles.
TRM-based methods select with the largest $q_{\mathrm{halt}}$ score, while GRAM uses its native selector.

\paragraph{Maze-Hard.}
Maze-Hard contains 1,000 training mazes and 1,000 test mazes.
TRM training uses the eight dihedral views, while FPRM adaptation uses unaugmented mazes.
TRM-based evaluations use reasoning depth 16.
PTRM uses the largest $q_{\mathrm{halt}}$ score, and W-PTRM applies the same rule after perturbing the Q head with the other trainable parameters.
FBI retains the first posterior candidate unless the highest-scoring candidate exceeds its logit by more than $\ln 2$.
FBI uses this margin rule with both TRM and FPRM backbones.
We also tested this rule with the baselines, but retained their existing selectors because it did not improve their performance.

\paragraph{ARC-AGI.}
We follow the augmentation-based ARC evaluation framework of TRM and PTRM.
ARC-AGI-1 and ARC-AGI-2 evaluation use the public ARC-AGI-1\footnote{\url{https://huggingface.co/datasets/arcprize/arc_agi_v1_public_eval}} and ARC-AGI-2\footnote{\url{https://huggingface.co/datasets/arcprize/arc_agi_v2_public_eval}} releases.
For each augmentation, we evaluate 25 fixed posterior draws at reasoning depth 16 and map their predictions back to the original grid orientation~~\citep{sghaier2026probabilistic}.
For ARC-AGI-1, we preserve its first two outputs from conservative Q- and structure-aware aggregation and rank the remainder by row-relative Q support within a one-logit margin.
For ARC-AGI-2, we additionally use row-relative Q support and structural consistency for its first two outputs.
We also tested FBI’s aggregation and ranking rules with the baselines but these changes did not improve their performance.
Following the TRM/PTRM ARC-AGI evaluation~\citep{sghaier2026probabilistic}, we compute Pass@$K$ over the top $K$ predictions after aggregation and ranking, rather than over raw solver outputs.

\subsection{Baseline details}
\label{app:baseline-details}

\paragraph{Sampling-based reasoning.}
\textbf{PTRM} keeps the model parameters fixed, injects Gaussian noise into the latent state at each deep-recursion step, and selects among the resulting trajectories.

\paragraph{Weight-sampling controls.}
\textbf{W-PTRM} and \textbf{W-FPRM} are non-learned weight-sampling baselines.
Let $\theta_0\in\mathbb R^d$ collect every trainable floating-point parameter of the corresponding checkpoint, including its Q head but excluding buffers and masks.
For candidate $j$, both controls sample the isotropic Gaussian
\begin{equation*}
  \theta^{(j)}=\theta_0+\sigma\epsilon_j,
  \qquad
  \epsilon_j\overset{\mathrm{iid}}{\sim}\mathcal N(0,I_d),
  \qquad
  \sigma=m\frac{1}{d}\sum_{r=1}^{d}|\theta_{0,r}|,
\end{equation*}
where $m$ is a dimensionless multiplier.
Each draw perturbs the original checkpoint independently, without latent noise.
\Cref{tab:weight-sampling-controls} gives the selected scales.
Sudoku-Extreme and Maze-Hard use scales selected on the validation splits described above and freeze them before test evaluation.
ARC-AGI-1 uses the fixed multiplier 0.30.
ARC-AGI-2 selects its scale by Pass@2, then Pass@1, Pass@5, and pairwise Q-score accuracy.

\begin{table}[htbp]
  \caption{\textbf{Fixed-Gaussian weight-sampling controls.}
  The variance of every sampled scalar is $\sigma^2$.}
  \label{tab:weight-sampling-controls}
  \centering
  \small
  \setlength{\tabcolsep}{7pt}
  \begin{tabular}{lcc}
    \toprule
    Setting & $m$ & $\sigma$ \\
    \midrule
    W-PTRM, Sudoku-Extreme & 0.30 & 0.00312003 \\
    W-PTRM, Maze-Hard & 0.30 & 0.00202273 \\
    W-PTRM, ARC-AGI-1 & 0.30 & 0.00650025 \\
    W-PTRM, ARC-AGI-2 & 0.20 & 0.00453306 \\
    W-FPRM, Maze-Hard & 0.05 & 0.000298018 \\
    \bottomrule
  \end{tabular}
\end{table}

\paragraph{Additional comparison with GRAM.}
\textbf{GRAM} models reasoning as stochastic latent trajectories and samples several trajectories in parallel~\citep{baek2026generative}.
\Cref{tab:gram-supplementary} compares GRAM and FBI on Sudoku-Extreme and Maze-Hard.

\begin{table}[htbp]
  \caption{\textbf{Additional recursive reasoning comparison.}
  Acc. is selected accuracy, and Pass@10 is coverage among ten candidates.
  All values are percentages.
  FBI exceeds GRAM on both metrics and benchmarks.}
  \label{tab:gram-supplementary}
  \centering
  \small
  \setlength{\tabcolsep}{6pt}
  \begin{tabular}{lcccc}
    \toprule
    & \multicolumn{2}{c}{Sudoku-Extreme} & \multicolumn{2}{c}{Maze-Hard} \\
    \cmidrule(lr){2-3}\cmidrule(lr){4-5}
    Method & Acc. $\uparrow$ & Pass@10 $\uparrow$ & Acc. $\uparrow$ & Pass@10 $\uparrow$ \\
    \midrule
    GRAM & 98.48 & 98.49 & 87.30 & 87.30 \\
    \rowcolor{fbrowhighlight} \textbf{FBI (Ours)} & \textbf{99.53} & \textbf{99.55} & \textbf{88.00} & \textbf{91.30} \\
    \bottomrule
  \end{tabular}
\end{table}

\subsection{Recursive posterior training}
\label{app:posterior-training}

\paragraph{Posterior optimization methods.}
We use the IVON posterior update described in \Cref{sec:ivon-update}~\citep{shen2024variational}.
For IVON-based methods, we initialize the posterior mean from a task checkpoint and fine-tune its mean and diagonal precision on the task training split.
Direct-10 uses the same failure score and IVON posterior machinery as FBI but estimates the $K_{\mathrm{train}}=10$ data term with ten independent posterior draws per update.
It optimizes the transformed failure score rather than the raw task loss.

\paragraph{Failure-score construction.}
Given a solver's nonnegative per-problem training loss $\ell$, FBI uses $f=1-\exp(-\ell/T)$ with a task-specific temperature $T>0$.
The loss $\ell$ is produced by the solver's training objective before auxiliary selector losses are added.
Write $\bar f_\phi(x)\coloneqq\mathbb E_{\theta\sim q_\phi}[f(\theta,x)]$. Replacing the binary failure probability in \Cref{eq:exact-finite-k-objective} with this posterior mean gives the score-based objective
\begin{equation}
  \min_\phi~
  \mathbb E_x[\bar f_\phi(x)^{K_{\mathrm{train}}}]
  +
  \tau_{\mathrm{KL}}D_{\mathrm{KL}}(q_\phi\Vert p_0).
  \label{eq:finite-k-score-objective}
\end{equation}
This objective equals the exact all-fail target only when $\bar f_\phi(x)$ equals the one-candidate failure probability; otherwise it is a surrogate.
For $K_{\mathrm{train}}$ independent posterior draws, $\prod_{k=1}^{K_{\mathrm{train}}} f(\theta_k,x)$ is an unbiased Monte Carlo estimator of $\bar f_\phi(x)^{K_{\mathrm{train}}}$.
These scores supply the fixed-weight IVON loss in \Cref{eq:fbi-weighted-loss}.

\paragraph{Practical posterior losses.}
Let $B_s$ be the minibatch at global optimizer step $s$, and let $g_i$ denote the problem or augmentation-group key for row $i$.
With the batch-averaged Q auxiliary loss denoted by $\mathcal L_q$, the implemented recursive FBI scalar differentiated at step $s$ is
\begin{equation}
  \mathcal L_{\mathrm{FBI},s}
  =
  \frac{1}{|B_s|}\sum_{i\in B_s}
  \lambda_{g_i,s} f(\theta_s,x_i)
  +\alpha_q\mathcal L_q(\theta_s;B_s).
  \label{eq:fb-practical-loss}
\end{equation}
Each $\lambda_{g_i,s}=K_{\mathrm{train}}\widehat f_{g_i,s}^{K_{\mathrm{train}}-1}$ is computed from the tracker before drawing $\theta_s$ and is held fixed and detached throughout that posterior update.
If several rows in the minibatch share a group key, their failure scores are averaged before one tracker update for that group; tracker states for absent groups remain unchanged.
The Fenchel penalty is omitted from this scalar because it is constant at fixed $\lambda$.
Under IVON's parameterization, ESS and weight decay jointly specify the zero-mean isotropic Gaussian prior $p_0$.
Their product gives its precision.
We vary the KL strength through ESS while adjusting weight decay inversely when the prior is to be held fixed.
For the TRM runs, $\mathcal L_q$ averages prefix-level, class-balanced cell-level, and smooth-min aggregate binary cross-entropy terms.
This auxiliary is added after constructing $f$, so it is neither part of the failure score nor multiplied by $\lambda$.
It contributes to both gradient and curvature estimates.

\paragraph{TRM training configurations.}
\Cref{tab:fb-trm-training,tab:arc-training} report the task-specific settings.
All four runs use curvature initialization 3.0, IVON momentum 0.9, curvature EMA 0.99999, and posterior noise scale 1.0.
FBI initializes each failure estimate as $\widehat f_x=0.5$.
On Sudoku-Extreme and Maze-Hard, dense parameters use IVON, puzzle embeddings use separate learning rates, and model EMA is 0.999.
ARC-AGI freezes the puzzle embeddings.
We use the TRM-MLP variant for Sudoku-Extreme and the TRM-Att variant for Maze-Hard, ARC-AGI-1, and ARC-AGI-2.

\begin{table}[ht]
  \caption{\textbf{Training configurations for TRM-based FBI.}
  The Sudoku-Extreme and Maze-Hard runs optimize $K_{\mathrm{train}}=10$ with one posterior draw per update.
  Both runs use constant learning rates and bfloat16 forward computation.
  Both use batch size 128, failure-score temperature 1.0, and puzzle-embedding weight decay 0.1.}
  \label{tab:fb-trm-training}
  \centering
  \small
  \setlength{\tabcolsep}{7pt}
  \begin{adjustbox}{max width=\linewidth}
  \begin{tabular}{lcc}
    \toprule
    Quantity & Sudoku-Extreme & Maze-Hard \\
    \midrule
    Updates & 200 & 188 \\
    Dense learning rate & $1\times10^{-4}$ & $2.8565\times10^{-5}$ \\
    Puzzle-embedding learning rate & $1\times10^{-2}$ & $1.3323\times10^{-4}$ \\
    IVON effective sample size & 30,000 & 15,000 \\
    IVON weight decay & 0.225 & 0.5 \\
    Relative KL strength & 1 & 2 \\
    $q_{\mathrm{halt}}$ loss weight & 0.35 & 0.27760 \\
    Failure-EMA coefficient & 0.99 & 0.9 \\
    \bottomrule
  \end{tabular}
  \end{adjustbox}
\end{table}

\begin{table}[ht]
  \caption{\textbf{ARC-AGI posterior-training configurations.}
  Both runs use a diagonal IVON posterior over 6,829,058 dense parameters, frozen puzzle embeddings, and one parameter draw per update.
  Both use $K_{\mathrm{train}}=10$, failure-score temperature 1.0, dense learning rate $10^{-4}$, and failure-EMA coefficient 0.99.}
  \label{tab:arc-training}
  \centering
  \small
  \setlength{\tabcolsep}{5pt}
  \begin{adjustbox}{max width=\linewidth}
  \begin{tabular}{lcc}
    \toprule
    Quantity & ARC-AGI-1 & ARC-AGI-2 \\
    \midrule
    Updates / batch size & 400 / 128 & 400 / 128 \\
    $q_{\mathrm{halt}}$ loss weight & 1.0 & 0.5 \\
    IVON effective sample size & 30,000 & 30,000 \\
    IVON weight decay & 0.225 & 0.225 \\
    Evaluation posterior scale & 1.5 & 1.5 \\
    \bottomrule
  \end{tabular}
  \end{adjustbox}
\end{table}

\paragraph{FPRM training configuration.}
FBI initializes its posterior mean from the EMA weights of the released FPRM checkpoint.
We initialize IVON from scratch with zero momentum and unit curvature, without a separate model EMA.
\Cref{tab:fb-fprm-training} lists the settings for 2,000 updates.

\begin{table}[ht]
  \caption{\textbf{FPRM-based FBI training on Maze-Hard.}
  The run uses bfloat16 forward computation, $K_{\mathrm{train}}=10$, and one posterior draw per update.}
  \label{tab:fb-fprm-training}
  \centering
  \small
  \setlength{\tabcolsep}{7pt}
  \begin{tabular}{lc}
    \toprule
    Quantity & Value \\
    \midrule
    Global batch size / GPUs & 768 / 8 \\
    Dense learning rate & $10^{-4}$ \\
    IVON effective sample size & 100,000 \\
    IVON momentum / curvature EMA & 0.9 / 0.999 \\
    IVON noise scale & 1.0 \\
    Failure-score temperature / EMA coefficient & 0.0625 / 0.95760 \\
    Initial failure estimate / $q_{\mathrm{halt}}$ loss weight & 0.5 / 0.5 \\
    \bottomrule
  \end{tabular}
\end{table}

\subsection{Autoregressive experiments}
\label{app:autoregressive-details}

\paragraph{Model and evaluation.}
We use Qwen3-4B-Base with MATH7.5K training data and evaluate on MATH500, OlymMATH, AMC23, AIME24, and AIME25.
Each FBI candidate uses an independent posterior parameter draw followed by autoregressive sampling.
The sampling temperature is 0.6 and top-$p$ is 0.95~\citep{zhang2026beyond}.

\paragraph{Evaluation protocol and metrics.}
We generate sixteen candidates per problem, score every candidate with Math-Verify~\citep{kydlicek2024mathverify}, and record binary correctness.
If $c$ of a problem's sixteen candidates are correct, we estimate Pass@8 as $1-\binom{16-c}{8}/\binom{16}{8}$.
We average the problem-level estimates within each benchmark and report the unweighted mean across the five benchmark scores.
All entries in \Cref{tab:autoregressive-standard} come from these common-protocol evaluations.
\Cref{tab:autoregressive-pass-at-k} reports the complete candidate-coverage comparison at $K\in\{1,4,8,16\}$, including the methods omitted from \Cref{fig:autoregressive-pass-at-k} for readability.

\begin{table}[ht]
  \caption{\textbf{Mathematical-reasoning candidate coverage with Qwen3-4B-Base.}
  Each entry is the unweighted mean Pass@$K$ across the five benchmarks, in percent.
  All methods use the same evaluation protocol and sixteen candidates per problem.}
  \label{tab:autoregressive-pass-at-k}
  \centering
  \small
  \setlength{\tabcolsep}{6pt}
  \begin{adjustbox}{max width=\linewidth}
  \begin{tabular}{lcccc}
    \toprule
    Method & Pass@1 & Pass@4 & Pass@8 & Pass@16 \\
    \midrule
    Base & 27.07 & 46.73 & 54.45 & 61.10 \\
    GRPO & 41.89 & 52.23 & 56.82 & 61.41 \\
    Pass@$K$ Training & 40.17 & 53.67 & 59.55 & \underline{64.16} \\
    DMPO & \textbf{43.01} & \textbf{55.03} & \underline{59.60} & 63.28 \\
    MaxRL & \underline{42.43} & 54.38 & 58.97 & 62.71 \\
    MMPO & 38.79 & 51.01 & 56.56 & 61.94 \\
    \rowcolor{fbrowhighlight} \textbf{FBI (Ours)} & 40.22 & \underline{54.52} & \textbf{60.44} & \textbf{66.06} \\
    \bottomrule
  \end{tabular}
  \end{adjustbox}
\end{table}

\paragraph{FBI training.}
The common posterior update is given in \Cref{sec:ivon-update}.
We place a diagonal IVON posterior over all actor parameters, initialize its mean from Qwen3-4B-Base, and set its diagonal curvature to $10^{-3}$.
We train for 300 steps with learning rate $0.5$, effective sample size $10^9$, batch size 32, and seed 42 for the data order and IVON noise.
Each update draws one parameter sample shared by its eight responses per prompt and all backward passes.
We estimate policy gradients with REINFORCE Leave-One-Out (RLOO)~\citep{ahmadian2024back}, which uses the other responses in each group to construct a leave-one-out reward baseline.
The stored failure score starts at $8^{-1/7}$ and yields the normalized $K_{\mathrm{train}}=8$ weight $8\hat f^7$ before each update.
After a successful update, an EMA with coefficient $0.5$ incorporates the observed failure $1-\frac{1}{8}\sum_{j=1}^{8}r_j$, where Math-Verify supplies binary rewards $r_j$.

\paragraph{FBI evaluation.}
Following MMPO~\citep{zhang2026beyond}, we use the following user prompt:
\begin{quote}
\small
\texttt{\{problem\} Let's think step by step and output the final answer within \textbackslash boxed\{\}.}
\end{quote}
We use a 4,096-token response limit.
We independently sample 16 parameter vectors from the learned posterior and use each vector to generate one response per problem.
We fix the posterior-draw and decoding seed schedules for reproducibility.

\subsection{Bayesian method comparison}
\label{app:bayesian-comparison}

\paragraph{Comparison details.}
The Bayesian comparison in Figure~\ref{fig:maze-bayesian-methods} follows the Maze-Hard/TRM evaluation protocol described in Appendix~\ref{app:recursive-evaluation}.
All methods start from the public TRM checkpoint.
IVON and EVON use single-candidate training, whereas FBI and FBE target $K_{\mathrm{train}}=10$ using one posterior draw per update.
The FBI training configuration is given in Appendix~\ref{app:posterior-training}.

\paragraph{Bayesian methods.}
IVON and FBI use diagonal Gaussian posteriors~\citep{shen2024variational}.
EVON and FBE learn a Gaussian posterior with diagonal covariance in the eigenbasis of a Kronecker-factored preconditioner~\citep{minut2026soap}.
Transforming this posterior back to the original parameter space yields a block-diagonal covariance that allows correlations within parameter blocks.
LA fits a diagonal empirical-Fisher approximation around the weights of the model~\citep{daxberger2021laplace}.
SWAG constructs a Gaussian posterior from the mean and diagonal-plus-low-rank covariance of training checkpoints~\citep{maddox2019swag}.

\subsection{Training compute accounting}
\label{app:compute-accounting}

\paragraph{Resource-benchmark setup.}
For resource measurement in \Cref{fig:maze-hard-training-efficiency}, FBI and Direct-10 use the same TRM backbone, initial checkpoint, dataset, seed, global batch size, bfloat16 computation, IVON schedule, and one NVIDIA RTX A6000 48\,GB GPU, with compilation disabled.
The benchmark changes only the finite-$K_{\mathrm{train}}$ gradient estimator and its number of posterior draws, one for FBI and ten for Direct-10, with $K_{\mathrm{train}}=10$ for both.
Both methods sample parameters from a learned IVON posterior.
The TRM experiments estimate IVON curvature with a Price's-theorem estimator.
FBI reuses a scaled finite-$K_{\mathrm{train}}$ gradient for curvature on most updates and computes an ordinary full-batch curvature anchor every 20 updates.
Direct-10 sequentially evaluates ten independent parameter draws, retains their computation graphs, and computes the product gradient with one forward and one backward pass per draw.
Direct-10 also computes a full-batch curvature anchor every 20 updates.
The Pass@10 values in \Cref{fig:maze-hard-training-efficiency} are the reported repeated-run endpoint means, retained separately from this resource benchmark.

\paragraph{Timing, FLOPs, and memory.}
The timed loop includes batch retrieval, optimizer updates, and EMA.
PyTorch's FLOP counter profiles ordinary and curvature-anchor updates, combined in their 19:1 frequency.
It counts matrix multiplication and fused attention in forward and backward passes, but not every elementwise operation.
Peak allocated GPU memory is measured across the 30-update benchmark run, with the peak counter reset before model construction.
\Cref{tab:maze-compute-accounting} summarizes the resource measurements used in \Cref{fig:maze-hard-training-efficiency}.

\begin{table}[ht]
  \caption{\textbf{Matched one-GPU training-resource benchmark on Maze-Hard.}
  Time and counted FLOPs are per update.
  FBI uses 12.3\% of Direct-10's time and 10.4\% of its counted FLOPs.}
  \label{tab:maze-compute-accounting}
  \centering
  \small
  \setlength{\tabcolsep}{7pt}
  \begin{tabular}{lcc}
    \toprule
    Quantity & FBI & Direct-10 \\
    \midrule
    Posterior draws per update & 1 & 10 \\
    Training time/update (s) & 0.0692 & 0.5614 \\
    Training FLOPs/update ($10^{12}$) & 0.427 & 4.087 \\
    Peak allocated GPU memory (GB) & 0.611 & 3.707 \\
    \bottomrule
  \end{tabular}
\end{table}

\paragraph{Latent-space direct $K=10$ training.}
We additionally study a latent-space direct $K=10$ training variant as an auxiliary performance and resource comparison.
This variant injects independent Gaussian noise into TRM's reasoning latent and uses the product of ten candidate failure scores as its direct-$K$ data term under shared model weights.
We refer to it as latent-space direct $K=10$ training.
On one GPU, all ten candidates are processed in a single batch.
For the resource benchmark, we retain the same backbone, initial checkpoint, dataset, seed, bfloat16 precision, and disabled compilation as above, with original data batch size 1 and $K_{\mathrm{train}}=10$.

\begin{table}[th]
  \caption{\textbf{Latent-space direct-10: performance and one-GPU training resources.}
  Pass@10 reports means over three independent training runs under the same Maze-Hard evaluation protocol.
  Performance is measured separately from the one-GPU resource benchmark, whose GPU count is reported in the GPUs column.
  Time reports the mean and sample standard deviation across three benchmark repetitions.
  }
  \label{tab:maze-candidate-parallelism}
  \centering
  \small
  \setlength{\tabcolsep}{4pt}
  \begin{tabular}{lccccc|c}
    \toprule
    Method & GPUs & Time/update & \multicolumn{2}{c}{Memory (GB)} & FLOPs/update & Pass@10 \\
    & & (ms) & Max/GPU & Sum & ($10^{12}$) & (\%) \\
    \midrule
    FBI & 1 & $72.4_{\pm 4.6}$ & 0.611 & 0.611 & 0.427 & \textbf{91.3} \\
    Latent-space Direct-10 & 1 & $143.4_{\pm 0.7}$ & 2.595 & 2.595 & 4.067 & 88.8 \\
    \bottomrule
  \end{tabular}
\end{table}

\Cref{tab:maze-candidate-parallelism} compares endpoint performance and separately measured one-GPU training resources.
FBI achieves a mean Maze-Hard Pass@10 of 91.3\%, compared with 88.8\% for latent-space Direct-10 under the same evaluation protocol, a difference of 2.5 percentage points.
In the separate one-GPU resource benchmark, FBI uses 50.5\% of the measured update time, 10.5\% of the counted FLOPs, and 23.5\% of the peak allocated GPU memory of the latent-space variant.
On the other hand, FBI and parameter-space Direct-10 learn both the posterior mean and parameter-wise sampling scales.
This allows the magnitude of weight perturbations to adapt to the training data, providing a learned mechanism for controlling candidate variation.

\section{Hyperparameter sensitivity}
\label{app:ablations}

We vary the failure-EMA coefficient and relative KL strength used by FBI on Maze-Hard, holding relative KL strength at 2 for the EMA sweep and the failure-EMA coefficient at its reported value for the relative-KL sweep.
We use the same random seed for all experiments in these sweeps.
For the relative-KL sweep, we vary $\tau_{\mathrm{KL}}$ through IVON's effective sample size while adjusting weight decay inversely.
This keeps the product of ESS and weight decay, and hence the implied Gaussian prior, fixed.

\begin{figure}[ht]
  \centering
  \includegraphics[width=\linewidth]{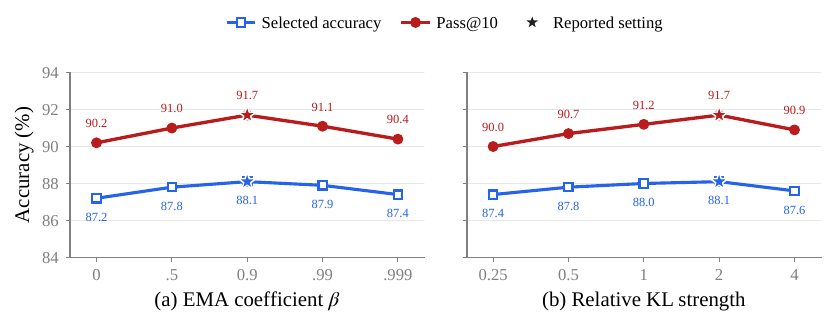}
  \caption{\textbf{EMA and KL sensitivity of FBI on Maze-Hard.}
  (a) varies the failure-EMA coefficient $\beta$ at relative KL strength 2.
  (b) varies relative KL strength at $\beta=0.9$ while keeping the implied prior fixed.
  Selected accuracy uses the frozen prefix-Q selector, and Pass@10 measures coverage among ten posterior candidates on test set with one fixed seed.
  Both metrics are highest at $\beta=0.9$ in (a) and relative KL strength 2 in (b) among the tested settings.}
  \label{fig:maze-hard-ema-kl-ablation}
\end{figure}

\Cref{fig:maze-hard-ema-kl-ablation} evaluates selected accuracy and Pass@10 under the same prefix-Q selector as the main Maze-Hard result.
Across the EMA sweep at relative KL strength 2, selected accuracy ranges from 87.2\% to 88.1\% and Pass@10 from 90.2\% to 91.7\%.
Both metrics rise from $\beta=0$ to the reported $\beta=0.9$ setting, then fall at $\beta=0.99$ and $0.999$.
At $\beta=0.9$, selected accuracy is 88.1\% and Pass@10 is 91.7\%.
The $\beta=0$ setting is a boundary control outside the finite-$\eta$ Bregman-proximal setting analyzed above.
In the relative-KL sweep, both metrics rise from 87.4\%/90.0\% at strength 0.25 to 88.1\%/91.7\% at strength 2, then fall to 87.6\%/90.9\% at strength 4.

\end{document}